\documentclass[12pt, a4paper]{amsart}

\usepackage{amsmath}
\usepackage{amssymb}

\newtheorem{thm}{Theorem} 
\newtheorem{lem}[thm]{Lemma} 
\newtheorem{prop}[thm]{Proposition} 
\newtheorem{definition}[thm]{Definition}
\newtheorem{cor}[thm]{Corollary}

\newtheorem{exm}[thm]{Example}
\newtheorem{rmk}[thm]{Remark}

\usepackage{subfigure}
\usepackage{tikz}
\usepackage[all]{xy}
\usetikzlibrary{patterns}
\usetikzlibrary{snakes}
\usepackage{enumerate}
\usepackage{multicol}
\usepackage{multirow}
\usepackage{booktabs}
\usepackage{color}

\usepackage{float}
\usepackage{caption}

\newcommand{\cerouno}{{[0,1]}}
\DeclareMathOperator{\Hom}{Hom}
\DeclareMathOperator{\FS}{FS}
\DeclareMathOperator{\SVFS}{SVFS}
\DeclareMathOperator{\IVFS}{IVFS}
\DeclareMathOperator{\CVFS}{CVFS}
\DeclareMathOperator{\T2FS}{T2FS}
\DeclareMathOperator{\HFS}{HFS}
\DeclareMathOperator{\THFS}{THFS}
\newcommand{\uno}{\mathbf{1}}
\newcommand{\dos}{\mathbf{2}}

\newcommand{\powerset}[1]{\dos^{#1}}
\newcommand{\Lat}{{\mathcal{L}at}}

\title[Lattice Embeddings between types of fuzzy sets]{Lattice embeddings between types of fuzzy sets. Closed-valued fuzzy sets}
\author[Lobillo]{F. J. Lobillo}
\author[Merino]{Luis Merino}
\author[Navarro]{Gabriel Navarro}
\thanks{The authors are listed in alphabetical order. All the authors contributed
equally to this work.}
\thanks{Corresponding author (Gabriel Navarro, email: gnavarro@ugr.es)}
\author[Santos]{Evangelina Santos}

\address[F. J. Lobillo]{Department of Algebra and CITIC, University of Granada}
\address[Luis Merino]{Department of Algebra, University of Granada}
\address[Gabriel Navarro]{Department of Computer Science and Artificial Intelligence and CITIC, University of Granada}
\address[Evangelina Santos]{Department of Algebra, University of Granada}
\email{jlobillo@ugr.es, lmerino@ugr.es, gnavarro@ugr.es, esantos@ugr.es}
\thanks{Gabriel Navarro has been partially supported by grant TIN2013-41990-R from the Ministerio de Econom\'{\i}a y Competitividad  of the Spanish Government and from Fondo Europeo de Desarrollo Regional FEDER}

\date{}

\begin{document}
\maketitle

\begin{abstract}
 In this paper we deal with the problem of extending Zadeh's operators on fuzzy sets (FSs) to interval-valued (IVFSs), set-valued (SVFSs) and type-2 (T2FSs) fuzzy sets. Namely, it is known that seeing  FSs  as SVFSs, or T2FSs, whose membership degrees are singletons is not order-preserving. We then describe a family of lattice embeddings from FSs to SVFSs. Alternatively, if the former singleton viewpoint is required, we reformulate the intersection on hesitant fuzzy sets and introduce what we have called closed-valued fuzzy sets.  This new type of fuzzy sets extends standard union and intersection on FSs. In addition, it allows handling together membership degrees of different nature as, for instance, closed intervals and finite sets. Finally, all these constructions are viewed as  T2FSs forming a chain of lattices.
\end{abstract}

\section{Introduction}
The relevance of the different types of fuzzy sets, as models to handle uncertainty, is worldwide recognized. This is yet another example of the continuing need for mathematical formalisms in computer science-related fields. Coming from Zadeh's notion of Fuzzy Set ($\FS$) \cite{Zadeh1965}, some of these types have been also defined in order to solve  the problem of the attribution of membership degree: whenever the objects of the universe of discourse are connected to imprecise concepts, the design the membership functions might be hard to achieve. Nevertheless, as solutions to specific problems of different nature, their algebraic structures seem to be disconnected or, at least, their relationships are not well-understood.

In the literature there exist some attempts aiming to establish a common framework where these types of fuzzy sets could be inserted gradually into a chain of lattice structures, see the recently published survey \cite{Bustince/etal:2016} and the references therein.  Nevertheless, for some cases, the compatibility between lattice structures and embedding maps is a tricky problem.  The origin of our work is motivated by the problem stated in \cite[Remark 2]{Bustince/etal:2016}: is there a lattice structure on the class of Set-Valued Fuzzy Sets ($\SVFS$s) \cite{Grattan1976} extending Zadeh's  union and intersection on $\FS$s? Actually, this question can also be stated in the realm of Type-2 Fuzzy Sets ($\T2FS$s) \cite{Zadeh1975}. The key-point here is how we may see $\FS$s as $\SVFS$s, or as $\T2FS$s.  In general, $\FS$s are viewed as $\SVFS$s for which membership degrees are given by a singleton on $[0, 1]$. This is not consistent with the standard lattice operators, and this is why it is proposed to change the algebraic structure.  Nevertheless, from an algebraic point of view, what is important  are the relations (lattice morphisms) between the objects (lattices) of the category.  That is, a different approach to a practical solution could then be to change the embedding and keep the lattice structures.

In this paper we study the embeddings between the lattice structures on some types of fuzzy sets. Concretely, $\FS$s and the generalizations given by Interval-Valued Fuzzy Sets ($\IVFS$s) \cite{Sambuc1975}, $\SVFS$s and $\T2FS$s. Our main goal is to present two possible solutions to the above question. One keeps the standard lattice structures whilst the other keeps the usual embedding from $\FS$s to  $\SVFS$s. In particular, the aims of our work are the following:
\begin{enumerate}
\item To show, from a mathematical perspective, why Zadeh's lattice operators on $\FS$s and set theoretical-derived operators on $\SVFS$s are not compatible by way of the embedding from $[0,1]$ into its power set as singletons.
\item To provide a collection of embeddings from $\FS$s to $\SVFS$s making these structures compatible.
\item To reformulate the meet operator on Hesitant Fuzzy Sets ($\HFS$s) \cite{Torra2010} in order to obtain a lattice structure for most of their practical applications.
\item To introduce Closed-Valued Fuzzy Sets ($\CVFS$s) in order to extends the lattice of $\FS$s via the embedding from $[0,1]$ into its power set as singletons.
\item To embed all these classes into $\T2FS$s forming a chain of lattice embeddings.
\end{enumerate}

The paper is organized in increasing order of generality of the different types of fuzzy sets. Namely, in Section \ref{fuzzy} we recall the notion of $\FS$ and explain the recurrent and basic categorical constructions of the paper. Readers non-familiar with category theory may consult, for instance, the reference \cite{Pierce1991}. Section \ref{IVFS} concerns $\IVFS$s. We then provide a family of embeddings from $\FS$s to $\IVFS$s respecting the usual meet and join operators. Sections \ref{SVFS} and \ref{CVFS} deal with the problem stated in \cite[Remark 2]{Bustince/etal:2016}. In the former we describe a family of lattice embeddings from $\FS$s to $\SVFS$s, where $\SVFS$s are endowed with the lattice operators derived from the set union and intersection. In the latter we show an alternative solution. We define the class of $\CVFS$s with a lattice structure inspired by the celebrated notion of $\HFS$. Then $\CVFS$s become a lattice which extends the lattice of $\FS$s through the ``singleton embedding''. Finally, in Section \ref{T2FS}, we embed the achieved chains of lattices inside the class of $\T2FS$s.

Due to the number of lattice structures we shall manage, we condense the notation used for the partial orders and the lattice operators in Table \ref{tablenot}. 
\begin{table*}[t]
  \centering 
   \renewcommand{\arraystretch}{1.5}
   \renewcommand{\thefootnote}{\fnsymbol{footnote}}
   \renewcommand\footnoterule{}
    \caption{Partial orders and lattice structures in this paper}\label{tablenot}
    \begin{minipage}{10cm}
  \begin{tabular}{lcccclccc}
  \toprule
  \multicolumn{4}{c}{Classes constructed from $[0,1]$} &  & \multicolumn{4}{c}{Types of fuzzy sets}  \\   \cmidrule{1-4}  \cmidrule{6-9}
   Class & Order & Join &Meet  & & Type & Order & Join &Meet  \\      \cmidrule{1-4}  \cmidrule{6-9}
$[0,1]$ & $\leq$   & $\vee$ & $\wedge$ & & $\FS(X)$ & $\leq_F$ & $\cup_F$ & $\cap_F$ \\
$\mathcal{I}([0,1])$ & $\leq_I$   & $\cup_I$ & $\cap_I$ & & $\IVFS(X)$ & $\sqsubseteq_I$ & $\sqcup_I$ & $\sqcap_I$ \\
$\mathcal{C}([0,1])$ & $\leq_C$   & $\cup_C$ & $\cap_C$ & &$\CVFS(X)$ & $\sqsubseteq_C$ & $\sqcup_C$ & $\sqcap_C$ \\
$\powerset{[0,1]}$ & $\subseteq$   & $\cup$ & $\cap$ & &\text{$\SVFS(X)$}\footnote[1]{Empty set, as membership degree, allowed} & $\sqsubseteq$ & $\sqcup$ & $\sqcap$ \\
$\powerset{[0,1]}\backslash \emptyset$\footnote[2]{\label{f1} It is only a partially ordered set, since $\cap_S$ is not a meet operator} & $\leq_S$   & $\cup_S$ & $\cap_S$ & &$\SVFS(X)$ & $\sqsubseteq_S$ & --  & -- \\ 
$[0,1]^{[0,1]}$ & $\leq_F$   & $\cup_F$ & $\cap_F$ & & $\T2FS(X)$ & $\sqsubseteq_{T2}$ & $\sqcup_{T2}$ & $\sqcap_{T2}$ \\ \bottomrule
  \end{tabular}
  \end{minipage}
\end{table*}
Throughout the paper we shall use the following notation. By \(\cerouno\) we denote the closed unit interval, \(\uno = \{\ast\}\) is the singleton, i.e. the set with one element, and \(\dos = \{0,1\}\) the two-element Boolean algebra. The set of maps from \(X\) to \(Y\) is denoted by \(Y^X\) or \(\Hom(X,Y)\); we use the second one when we want to highlight the categorical viewpoint. The power set of \(X\) is therefore denoted by \(\powerset{X}\) via identification with the characteristic maps.

\section{Fuzzy sets}\label{fuzzy}

In 1965, inspired by the idea of multivalued logic, Zadeh \cite{Zadeh1965} proposes the theory of Fuzzy Sets, an extension of the classical set theory, for some classes of objects whose criteria of membership are not precisely defined. In what follows, $X$ denotes a nonempty universe set. 

\begin{definition} A Fuzzy Set (FS) \(A\) on \(X\) is a  set mapping \(A : X \to \cerouno\). This is also called a Type-1 Fuzzy Set (T1FS). \end{definition}
That is, fuzzy sets on \(X\) can be described as 
$\FS(X) = \Hom(X,\cerouno)$,
the class of set maps from \(X\) to \(\cerouno\). 
The standard order $\leq$ on the unit interval \(\cerouno\) can be lifted pointwise to a partial order on \(\FS(X)\). For any \(A,B \in \FS(X)\) 
\[
A \leq_F B \iff A(x) \leq B(x) \text{ for all \(x \in X\)}.
\]
This inherited partial order is a particular case of a categorical construction. For an arbitrary nonempty set $S$, we may consider  the covariant endofunctor 
\begin{equation}\label{func}
\Hom(S,-):\mathcal{P}oset \to \mathcal{P}oset
\end{equation}
of the category of partially ordered sets (posets) which maps any poset $(P,\leq )$ to $\Hom(S,P)$ endowed with the partial order $\sqsubseteq$ defined by, for each $f,g\in \Hom(S,P)$, \[
f \sqsubseteq g \iff f(s) \leq g(s) \text{ for all \(s \in S\)}.
\]
The corresponding meet and join operators, minimum $\wedge$ and maximum $\vee$, associated to the order on [0,1] provide a lattice structure on $\FS(X)$ given by
\[
(A \cup_F B)(x) = A(x) \vee B(x) \text{ and } 
(A \cap_F B)(x) = A(x) \wedge B(x),
\]
for each $A,B\in \FS(X)$ and any $x\in X$. In general, for any lattice $(L,\wedge,\vee)$ and any nonempty set $S$, $\Hom(S,L)$ becomes a lattice with meet and join operators defined pointwise, i.e.
$$
(f \sqcup g)(s) = f(s) \vee g(s) \text{ and }
(f \sqcap g)(s) = f(s) \wedge g(s),
$$
for any $f,g \in \Hom(S,L)$ and any $s\in S$. This simply means that (\ref{func}) can be restricted to an endofunctor
\begin{equation}\label{func2}
\Hom(S,-):\mathcal{L}at \to \mathcal{L}at
\end{equation}
on the category of lattices. In addition, if $S$ is a chain (or totally ordered set), the set formed by all lattice maps $\Hom_{\mathcal{L}at}(S,L)$ from $S$ to $L$ becomes a lattice by restricting the operators $\sqcup$ and $\sqcap$. In such a case we may also consider the subfunctor
\begin{equation}\label{func3}
\Hom_{\mathcal{L}at}(S,-):\mathcal{L}at \to \mathcal{L}at.
\end{equation}

The practical utility of $\FS$s depends heavily on a suitable estimation of the membership degrees and, in some cases, it becomes a hard problem. Zadeh himself realizes this point in \cite{Zadeh1975} and suggests to make use of more complex objects as membership degrees rather than a solely real value. Several extensions (types) of FSs have arisen from this idea. Mainly, the generalizations invoke implicitly the $\Hom$ functor of (\ref{func2}), or the $\Hom_{\mathcal{L}at}$ functor of (\ref{func3}), for a lattice $L$ constructed from $[0,1]$. In the forthcoming sections we describe and relate those concerning this work.

\section{Interval-Valued Fuzzy Sets}\label{IVFS}

A closed interval in \(\cerouno\) is given by a couple of elements \(\alpha, \beta \in \cerouno\) such that \(\alpha \leq \beta\). So the set of closed intervals is in one-to-one correspondence with the set of lattice maps \(\Hom_{\Lat}(\dos,\cerouno)\).  Here, each \(f \in \Hom_{\Lat}(\dos,\cerouno)\) is identified with the interval $[f(0),f(1)]$. We shall denote by $\mathcal{I}([0,1])$ the class of all closed intervals in $[0,1]$. The idea of assigning a closed interval as membership degree can be found in Zadeh's paper \cite{Zadeh1975}, although it is also introduced independently by Grattan-Guiness \cite{Grattan1976}, Jahn \cite{Jahn1975} and Sambuc \cite{Sambuc1975}. Many practical applications related to this concept have been developed, see e.g. \cite[Section VI]{Bustince/etal:2016}. Formally:
\begin{definition}

An Interval-Valued Fuzzy Set ($\IVFS$) $A$ on \(X\) is a mapping $A : X \to \mathcal{I}([0,1])$, i.e. the set of all $\IVFS$s on $X$ is
$\IVFS(X) = \Hom\big(X,\Hom_{\Lat}(\dos,\cerouno)\big)$.
\end{definition}

As pointed out in Section \ref{fuzzy}, there exists a partial order $\sqsubseteq_I$ on IVFS\((X)\) inherited from a partial order $\leq_I$ on \(\Hom_{\Lat}(\dos,\cerouno)\), which is lifted pointwise from the one on [0,1]. Concretely, given two closed intervals $[a,b]$ and $[c,d]$, 
$$[a,b]\leq_I [c,d] \iff a\leq c \text{ and } b\leq d$$
and, given $A,B\in \IVFS(X)$, 
$$A\sqsubseteq_I B \iff A(x) \leq_I B(x) \text{ for any $x\in X$}.$$ 

The lattice structure on $\IVFS(X)$ comes from the lattice structure on \(\Hom_{\Lat}(\dos,\cerouno)\), which is lifted from  \(\cerouno\). Namely, for each closed intervals $[a,b]$ and $[c,d]$,
$$[a,b] \cup_I [c,d] = [a\vee c, b\vee d]  \text{ and } 
\text{$[a,b]   \cap_I [c,d] = [a \wedge c, b\wedge d]$},
$$
and then, for each $A,B\in \IVFS(X)$,
$$
(A \sqcup_I B)(x) = A(x) \cup_I B(x) \text{ and } 
(A \sqcap_I B)(x) = A(x) \cap_I B(x),
$$
 for any $x\in X$. These lattice structures are considered, for example, in \cite{Barrenechea2011}, \cite{Bustince1995}, \cite{Dubois1979} (by using the extension
principle applied to $\vee$ and $\wedge$) or \cite{Dubois2005}.
From a categorical perspective, any one-to-one lattice map $$h:([0,1],\vee,\wedge)\to (\Hom_{\mathcal{L}at}(\dos,[0,1]),\cup_I,\cap_I)$$ yields a lattice embedding $$\Hom(X,h):(\FS(X),\cup_F,\cap_F) \to (\IVFS(X),\sqcup_I,\sqcap_I)$$ by applying the functor (\ref{func2}) to $h$. Concretely, $\Hom(X,h)(A)=h\circ A$ for each $A\in \FS(X)$.
This is an embedding because (\ref{func2}) preserves monomorphisms, see \cite{Pierce1991}. 
Here, it is interesting to see $\Hom_{\mathcal{L}at}(\dos,[0,1])$ as a sublattice of $\Hom(\dos,[0,1])\equiv [0,1]^2$, the unit square in $\mathbb{R}^2$, with the product order. Therefore, $$\Hom_{\mathcal{L}at}(\dos,[0,1])\equiv \{(x,y)\in [0,1]^2 \text{ such that } x\leq y\}=T.$$
So, in order to find a lattice embedding from $\FS$s to $\IVFS$s with respect to the above-mentioned lattice structures, it is enough to fix a lattice embedding from $[0,1]$ to the upper triangle $T$. A simple way of finding one of such an embedding is fixing two increasing maps $$h_1,h_2:[0,1]\to [0,1]$$ such that $h_1(t)\leq h_2(t)$ for any $t\in [0,1]$, and one of them is strictly increasing.  The embedding $h$ is then defined as  $h(t)=(h_1(t),h_2(t))$ for any $t\in [0,1]$. Moreover, for any $A\in \FS(X)$ and any $x\in X$, the associated closed interval is $[h_1(A(x)),h_2(A(x))]$. For instance, set $h_1,h_2,h_3:[0,1]\to [0,1]$ defined by $h_1(t)=t$, $h_2(t)=0$ and $h_3(t)=1$ for any $t\in [0,1]$. Hence, we obtain the lattice embeddings
$$\phi,\omega , \gamma:([0,1],\vee,\wedge) \to (\mathcal{I}([0,1]),\cup_I,\cap_I)$$
given by $\phi(t)=[h_1(t),h_1(t)]=[t,t]$, $\omega(t)=[h_1(t),h_3(t)]=[t,1]$ and $\gamma(t)=[h_2(t),h_1(t)]=[0,t]$
for any $t\in [0,1]$, see Figure \ref{Fig2}(a)--(c).
\begin{figure*}[tb]
\begin{center}
\subfigure[Embedding $\phi$]{
\begin{tikzpicture}[scale=2.9]
\draw[dashed, ->] (0,-0.1) -- (0,1.1); 
\draw[dashed, ->] (-.1,0) -- (1.1,0); 
\node at (-.15,0.5) {\scriptsize{$t$}};
\draw[black, densely dotted] (0.5,0)--(0.5,0.5);
\draw[black, densely dotted] (-.1,0.5) -- (0.5,0.5);
\node at (0.5,-.07) {\scriptsize{$t$}};
\node at (-.05,1) {\scriptsize{$1$}};
\node at (1,-0.05) {\scriptsize{$1$}};
\node at (-.05,-0.05) {\scriptsize{$0$}};
\draw (0,0) -- (1,1); 
\draw[dotted] (0,1) -- (1,1); 
\draw[black, line width=1.5pt] (0,0) -- (1,1);
\draw (0.5,0.5) node [circle,fill, inner sep=1.5pt]{};
\end{tikzpicture}} \,
\subfigure[Embedding $\omega$]{
\begin{tikzpicture}[scale=2.9]
\draw[dashed, ->] (0,-0.1) -- (0,1.1); 
\draw[dashed, ->] (-.1,0) -- (1.1,0); 
\node at (-.15,1) {\scriptsize{$1$}};
\node at (1,-0.05) {\scriptsize{$1$}};
\node at (-.05,-0.05) {\scriptsize{$0$}};
\draw[dotted] (0,0) -- (1,1); 
\node at (0.5,-.07) {\scriptsize{$t$}};
\node at (-.15,0.5) {\scriptsize{$t$}};
\draw[black, densely dotted] (0.5,0)--(0.5,1);
\draw[black, densely dotted] (-.1,0.5) -- (0.5,0.5);
\draw[black, densely dotted] (-.1,1) -- (0.5,1);
\draw[black, line width=1.5pt] (0,1) -- (1,1);
\draw (0.5,1) node [circle,fill, inner sep=1.5pt]{};
\draw[<->] (-.05,0.51) -- (-.05,0.99);
\end{tikzpicture}}
\subfigure[Embedding $\gamma$]{
\begin{tikzpicture}[scale=2.9]
\draw[dashed, ->] (0,-0.1) -- (0,1.1); 
\draw[dashed, ->] (-.1,0) -- (1.1,0); 
\node at (-.05,1) {\scriptsize{$1$}};
\node at (1,-0.05) {\scriptsize{$1$}};
\node at (-.15,0) {\scriptsize{$0$}};
\draw[black, densely dotted] (0.5,0)--(0.5,0.5);
\draw[black, densely dotted] (-.1,0.5) -- (0.5,0.5);
\node at (0.5,-.07) {\scriptsize{$t$}};
\node at (-.15,0.5) {\scriptsize{$t$}};
\draw[dotted] (0,0) -- (1,1); 
\draw[dotted] (0,1) -- (1,1); 
\draw (0,0) -- (0,1); 
\draw[black, line width=1.5pt] (0,0) -- (0,1);
\draw (0,0.5) node [circle,fill, inner sep=1.5pt]{};
\draw[<->] (-.05,0.01) -- (-.05,0.49);
\end{tikzpicture}
}
\subfigure[Generic embedding derived from the graph of a set map $f$]{
\begin{tikzpicture}[scale=2.9]
\draw[dashed, ->] (0,-0.1) -- (0,1.1); 
\draw[dashed, ->] (-.1,0) -- (1.1,0); 
\node at (-.05,1) {\scriptsize{$1$}};
\node at (1,-0.05) {\scriptsize{$1$}};
\node at (-.05,-0.05) {\scriptsize{$0$}};
\draw[dotted] (0,0) -- (1,1); 
\draw[dotted] (0,1) -- (1,1); 
\draw[black, line width=1.5pt] plot [smooth] coordinates {(0,0) (0.3,0.7) (1,1)};
\draw[black, densely dotted] (-.1,0.3) -- (0.3,0.3);
\draw[black, densely dotted] (-.1,0.7) -- (0.3,0.7);
\draw[black, densely dotted] (0.3,0) -- (0.3,0.7);
\node at (-.15,0.3) {\scriptsize{$t$}};
\node at (-.2,0.7) {\scriptsize{$f(t)$}};
\draw (0.3,0.7) node [circle,fill, inner sep=1.5pt]{};
\node at (0.3,-0.05) {\scriptsize{$t$}};
\draw[<->] (-.05,0.31) -- (-.05,0.69);
\end{tikzpicture}
}
\end{center}
\caption{Lattice embeddings from $[0,1]$ to $\mathcal{I}([0,1])$.}\label{Fig2}
\end{figure*}
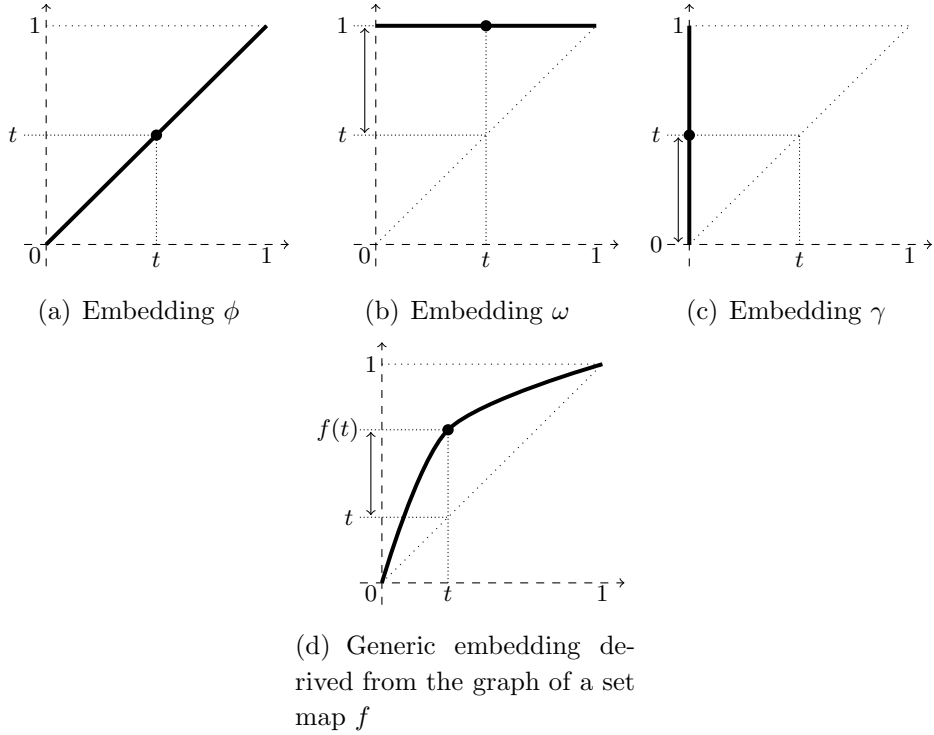
The first two examples are instances of the same construction. Namely, we may always consider the graph of an increasing set map $f:[0,1]\to[0,1]$ such that $f(t)\geq t$. In this case, the embedding is defined by $h(t)=(t,f(t))$ for any $t\in [0,1]$, see Figure \ref{Fig2}(d). These maps induce the corresponding lattice embeddings 
\[\Phi, \Omega, \Lambda : (\FS(X),\cup_F,\cap_F) \longrightarrow (\IVFS(X),\sqcup_I,\sqcap_I)\]
defined pointwise as $\Phi(A)(x)=[A(x),A(x)]$, $\Omega(A)(x)=[A(x),1]$ and $\Lambda(A)(x)=[0,A(x)]$ for any $A\in \FS(X)$ and any $x\in X$.  
 Observe that $\Phi$ is the inclusion of $\FS$s into $\IVFS$s pointwise-derived from seeing an element in $[0,1]$ as a singleton interval.

\section{Set-Valued Fuzzy Sets}\label{SVFS}

The idea of representing the plausible range of the membership degrees by means of closed intervals has achieved a relative success, see \cite[Section 4]{Bustince2010} and the references therein. An obvious generalization consists in allowing to choose an arbitrary nonempty subset of $[0,1]$, yielding the notion of set-valued fuzzy set. The following definition is given in \cite{Grattan1976}.
\begin{definition}\label{svfs}
A Set-Valued Fuzzy Set (SVFS) \(A\) on \(X\) is a set mapping $A : X \to \powerset{\cerouno}\backslash \{\emptyset\}$. Hence 
$
\SVFS(X) = \Hom\big(X,\powerset{\cerouno}\backslash \{\emptyset\}\big).
$
\end{definition}
The insertion of $\SVFS$s, as another piece of the chain of lattices of types of fuzzy sets, is somehow problematic.  Actually, in \cite[Remark 2]{Bustince/etal:2016} it is formulated the open problem of finding a lattice structure on the class of $\SVFS$s such that the restriction to $\FS$s preserves Zadeh's intersection and union. Let us explain the partial solutions that we propose in this work. In this section we consider $\SVFS$s endowed with the lattice structure derived from  set intersection and union, whilst, in the next one, we shall analyze another meet and join operators.

\subsection{Lattice embeddings from $\FS(X)$ to $\SVFS(X)$}

The power set of $[0,1]$ is endowed  with the usual lattice structure defined by the set union $\cup$ and intersection $\cap$. Hence, by virtue of (\ref{func2}), $\Hom(X,\powerset{\cerouno})$ becomes a lattice with the structure defined by the operators
$$
(A \sqcup B)(x) = A(x) \cup B(x) \text{ and } 
(A \sqcap B)(x) = A(x) \cap B(x),
$$
for each $A,B\in \Hom(X,\powerset{\cerouno})$ and any $x\in X$. Nevertheless, in general, $\SVFS(X)$ is not closed under these operators. Indeed, given $A,B\in \SVFS(X)$, there could exist an $x\in X$ such that $A(x)\cap B(x)=\emptyset$, the reader may find easily examples verifying this, so that $A\sqcap B$, according to Definition \ref{svfs}, is not in $\SVFS(X)$. This is a slight mistake committed in \cite[Proposition 5.1]{Bustince/etal:2016}, where it is asserted that $(\SVFS(X),\sqcup,\sqcap)$ is a complete lattice. 

An elementary solution consists in extending the notion of $\SVFS$ by allowing the empty set as a possible membership degree. A possible interpretation of this degree would be a nonsense \cite{Torra2009}. Hence, under this condition, $(\SVFS(X),\sqcup,\sqcap)$ is a complete lattice. 
Via the characteristic maps, this is to say we see $A$ and $B$ belonging to $\Hom(X,\Hom([0,1],\dos))$, we may write
$$
(A \sqcup B)(x)(t) = A(x)(t) \vee B(x)(t) \text{ and } 
(A \sqcap B)(x)(t) = A(x)(t) \wedge B(x)(t),
$$
for all  \(x \in X\) and \(t \in \cerouno\). So, this lattice structure on \(\SVFS(X)\) comes from the  Boolean algebra \(\dos\). This conclusion also follows from the natural isomorphism provided by the adjunction between the functors $\Hom([0,1],-)$ and $-\times [0,1]$,
\begin{equation}\label{adjunction}
\Hom\big(X,\powerset{\cerouno}\big) \cong \Hom(X \times \cerouno, \dos).
\end{equation}
As observed in \cite{Bustince/etal:2016}, $\FS$s can be viewed as $\SVFS$s via the injective map $\iota : \FS(X) \to  \SVFS(X)$ given by $\iota(A)(x)=\{A(x)\}$, for any $A\in \FS(x)$ and $x\in X$. This map is simply the composition
$$\xymatrix{\FS(X) \ar[r]^-{\Phi} & \IVFS(X) \ar[r]^-{i}  & \SVFS(X),
}$$
where $\Phi$ is the lattice embedding provided in Section \ref{IVFS} and $i$ is defined by lifting the identification of a closed interval as a subset of $[0,1]$.

By (\ref{func2}), the insertion of \(\cerouno\) inside \(\powerset{\cerouno}\) as singleton sets implies the inclusion
\[
\Hom(X,\cerouno) \subseteq \Hom(X,\powerset{\cerouno})
\]
and the map \(\iota\) is a description of this inclusion. Nevertheless, it is not a lattice injection. The isomorphism (\ref{adjunction}) 
explains why the lattice structures are not compatible: The lattice structure on \(\FS(X)\) is lifted from \(\cerouno\) applying (\ref{func2}), and the lattice structure on \(\SVFS(X)\) is inherited from \(\dos\). Actually, this is also the reason of why $i$ is not a lattice map. So, in order to make compatible the lattice structures, we have to look for a new way to embed \(\FS(X)\) inside \(\SVFS(X)\). 

Let \(A \in \FS(X)\), for any \(t \in \cerouno\), the level $t$-cut of $A$ is defined by 
$A_t= \{x \in X ~|~ A(x) \geq t\}$.
The class of all level $t$-cuts of $A$ is a chain, since \(t \leq s\) implies \(A_s \subseteq A_t\). Then $A$ can also be viewed as a lattice map 
\[A : (\cerouno^{\mathsf{op}},\wedge,\vee) \to (\powerset{X},\cup,\cap),\] 
where by \(\cerouno^{\mathsf{op}}\) we mean the set $[0,1]$ with the opposite order. Concretely, there is a bijective map
\begin{equation}\label{cutsident}
\begin{split}
c : \Hom(X,\cerouno) &\longrightarrow \Hom_{\Lat}(\cerouno^{\mathsf{op}},\powerset{X}) \\
A &\longmapsto \left[ t \mapsto A_t \right].
\end{split}
\end{equation}
Zadeh's lattice structure on $\FS(X)$ is reflected, via $c$, from the following one: for each \(A,B \in \Hom_{\Lat}(\cerouno^{\mathsf{op}},\powerset{X})\),
$$
(A \cup_{*} B)_t = A_t \cup B_t \text{ and }
(A \cap_{*} B)_t = A_t \cap B_t,
$$
for any $t\in [0,1]$. So, from this viewpoint, the lattice structure on \(\FS(X)\) is inherited from the lattice structure of the power set \(\powerset{X}\) and, whence, from $\dos$.  Since the map \(c\) is a bijective  map, the composition
$$
\begin{array}{c}
\FS(X)\overset{c}{\cong} \Hom_{\Lat}(\cerouno^{\mathsf{op}},\powerset{X}) \subseteq \Hom(\cerouno,\powerset{X}) 
\overset{(\ref{adjunction})}{\cong} \Hom(\cerouno \times X,\dos)\cong \\ \Hom(X\times \cerouno,\dos) \overset{(\ref{adjunction})}{\cong} \SVFS(X),\end{array}
$$
provides an embedding 
$$\Theta: (\FS(X),\cup_F,\cap_F)\to (\SVFS(X),\sqcup,\sqcap)$$
with $\Theta(A)(x)=[0,A(x)]$ for each $A\in \FS(X)$ and $x\in X$.
Observe that the lattice structures of all intermediate sets come from $\dos$. Therefore, as a composition of lattice maps, $\Theta$ is also a lattice map, and so, a lattice embedding. Nevertheless, for the convenience, here we show a direct proof.
\begin{thm}\label{FSinSVFS}
The map \(\Theta\) is a lattice embedding. 
\end{thm}
\begin{proof}
Assume first \(\Theta(A) = \Theta(B)\) for some $A,B\in \FS(X)$. Then, for all \(x \in X\), \([0,A(x)]=[0,B(x)],\) so \(A(x) = B(x)\). Thus $A=B$ and \(\Theta\) is an injective map.  Now, 
\[
\begin{split}
\Theta(A \cup_F B)(x) &= [0, (A \cup_F B)(x)] \\
&= [0, A(x) \vee B(x)] \\
&= [0,A(x)] \cup [0,B(x)] \\
&= \Theta(A)(x) \cup \Theta(B)(x)\\
& =(\Theta(A)\sqcup \Theta(B))(x),
\end{split}
\]
for any $x\in X$. Similarly, $\Theta(A \cap_F B)=\Theta(A) \sqcap \Theta(B)$. So
 \(\Theta\) is a lattice map.
\end{proof}
The above discussion gives an answer to the open problem stated in \cite[Remark 2]{Bustince/etal:2016}. Nevertheless, this solution does not include $(\IVFS(X),\sqcup_I,\sqcap_I)$ as a lattice
between $\FS(X)$ and $\SVFS(X)$. Although a closed interval is a subset of \(\cerouno\), and then $i:\IVFS(X)\to \SVFS(X)$  is a one-to-one map, the lattice operations on $\SVFS(X)$ are not closed when they are restricted to $\IVFS(X)$. For instance, the union of intervals is not always an interval. So the question is if there exists an embedding from $\IVFS(X)$ to $\SVFS(X)$ coherent with these lattice structures. That is, can we provide a commutative lattice diagram 
$$\xymatrix{(\FS(X),\cup_F,\cap_F) \ar[rd]_-{\Theta} \ar[r]^-{\Phi} & (\IVFS(X),\sqcup_I,\sqcap_I) \ar@{-->}[d]\\ & (\SVFS(X),\sqcup,\sqcap)? }$$
 Next, we give a positive answer.

\subsection{A lattice embedding from $\IVFS(X)$ to $\SVFS(X)$}\label{SVFSb}

Let us denote by $\mathbb{Q}$ and $\mathbb{I}$ the set of rational and irrational numbers, respectively. Consider the map $$\xi: \mathcal{I}([0,1]) \longrightarrow \powerset{[0,1]}$$ defined by $\xi([a,b])=([0,a]\cap \mathbb{Q})\cup ([0,b]\cap \mathbb{I})$ for any closed interval $[a,b]$. All along this section we shall need the following well-known properties concerning the real numbers, see for instance \cite[Theorem 2.6.13]{bloch2011}.
\begin{align}
&\begin{array}{l}
\textbf{Density of rationals}\text{. Given $a,b\in \mathbb{R}$ with $a<b$,}   \\ \text{there exists a rational $q\in \mathbb{Q}$ such that $a<q<b$.}
\end{array}
\label{densQ}
\\
&\begin{array}{ll}
\textbf{Density of irrationals}\text{. Given $a,b\in \mathbb{R}$ with $a<b$,}  \\ \text{there exists an irrational $p\in \mathbb{I}$ such that $a<p<b$.}
\end{array}
\label{densI}
\end{align}

\begin{lem}\label{density}
Let $I,J\in \mathcal{I}([0,1])$ such that $I$ is not a singleton, the following are equivalent:
\begin{enumerate}[$i)$]
\item $I\subseteq J$.
\item $I\cap \mathbb{Q}\subseteq J\cap \mathbb{Q}$.
\item $I\cap \mathbb{I}\subseteq J\cap \mathbb{I}$.
\end{enumerate}
\end{lem}
\begin{proof}
$i)\Rightarrow ii)$ and $i)\Rightarrow iii)$ are trivial. We prove $ii)\Rightarrow i)$. Let $I=[a,b]$ and $J=[c,d]$, with $a\not = b$, and $I\cap \mathbb{Q}\subseteq  J$. Let us first prove that $a<d$ and $c<b$. Since $a<b$, by (\ref{densQ}), there exists $q\in \mathbb{Q}$ such that $a<q<b$. Therefore $q\in I\cap \mathbb{Q}\subseteq J=[c,d]$, so $c\leq q <b$ and $a<q\leq d$. 

Now, we prove $c\leq a$ and $b\leq d$. Suppose that $a<c<b$. By (\ref{densQ}), there exists $q\in \mathbb{Q}$ such that $a<q<c$. Then $q$ is in $I\cap \mathbb{Q}$, but not in $J$. A contradiction, so $c\leq a$. Similarly, if $d<b$, there exists $q\in \mathbb{Q}$ such that $d<q<b$. Since $a<d$, $q\in I\cap \mathbb{Q}\subseteq J$, and we get a contradiction. Then $b\leq d$. Therefore $[a,b]\subseteq [c,d]$.

$iii)\Rightarrow i)$ can be proved similarly by applying (\ref{densI}).
\end{proof}
\begin{thm}
The map $\xi:(\mathcal{I}([0,1]),\cup_I,\cap_I)\to (\powerset{[0,1]},\cup,\cap)$ is a lattice embedding
\end{thm}
\begin{proof}
Let us fix $I=[a,b]$ and $J=[c,d]$ two closed intervals. Suppose first that $\xi([a,b])=\xi([c,d])$. Then, since $\mathbb{Q}$ and $\mathbb{I}$ are disjoint sets, 
$[0,a]\cap \mathbb{Q}=[0,c]\cap \mathbb{Q}$  and $[0,b]\cap \mathbb{I}=[0,d]\cap \mathbb{I}.$
If $a$ and $c$ are non zero, by Lemma \ref{density}, $[0,a]=[0,c]$, so $a=c$. Suppose one of them is zero. For instance, suppose that $a=0$. Hence $[0,c]\cap \mathbb{Q}=\{0\}$. By the density of $\mathbb{Q}$, it follows that  $c=0$. The equality $b=d$ may be proved similarly by applying Lemma \ref{density} and (\ref{densI}). Thus $\xi$ is an injective map. Now,
\[
\begin{split}
\xi(I)\cap \xi(J) &= (([0,a] \cap \mathbb{Q}) \cup ([0,b] \cap \mathbb{I})) \cap (([0,c] \cap \mathbb{Q}) \cup ([0,d] \cap \mathbb{I})) \\
&= ([0,a\wedge c]\cap \mathbb{Q})\cup ([0,b\wedge d]\cap \mathbb{I}) \\
 &= \xi([a\wedge c, b\wedge d]) \\
 &= \xi(I\cap_I J)
 \end{split}
\]
and, by a similar reasoning, $\xi(I)\cup \xi(J)=\xi(I \cup_I J)$. So $\xi$ is a lattice map.
\end{proof}

By applying the functor (\ref{func2}) to $\xi$, we get a lattice embedding
$$\Xi:(\IVFS(X),\sqcup_I,\sqcap_I)\to (\SVFS(X),\sqcup,\sqcap)$$
defined by $\Xi(A)=\xi\circ A$ for any $A\in \IVFS(X)$. Observe that the restriction to $\FS(X)$ maps an $A\in \FS(X)$ to $\Xi(A)\in \SVFS(X)$ defined by $x\mapsto [0,A(x)]$. Then, the following diagram of lattice maps is commutative:
$$\xymatrix{(\FS(X),\cup_F,\cap_F) \ar[rd]_-{\Theta} \ar[r]^-{\Phi} & (\IVFS(X),\sqcup_I,\sqcap_I) \ar[d]^-{\Xi}\\ & (\SVFS(X),\sqcup,\sqcap). }$$

\begin{cor}
The composition of $\Xi$ with any of the maps constructed in Section \ref{IVFS} provides a family of lattice embeddings from $(\FS(X),\cup_F,\cap_F)$ to $(\SVFS(X),\sqcup,\sqcap)$.
\end{cor}

\section{Closed-valued fuzzy sets}\label{CVFS}

Let us now develop an alternative solution to the problem in \cite[Remark 2]{Bustince/etal:2016}. In some practical situations it could be interesting  to insert FSs inside SVFSs via the  embedding $\iota$ defined in Section \ref{SVFS}. Therefore, in order to preserve Zadeh's operators on $\FS(X)$, we need to change the lattice structure on $\SVFS(X)$.  The most promising attempt in this direction is the notion of Hesitant Fuzzy Set ($\HFS$) defined by Torra in \cite{Torra2010}.  The class $\HFS(X)$ is constructed by  exchanging the lattice structure on $\SVFS(X)$ by the following operators:
$$\begin{array}{l}
(A \sqcup_H B)(x)  =  \{t \in A(x) \cup B(x) \text{ $|$ }  t\geq \underline{A(x)}\vee \underline{B(x)}\} \text{ and } \\
(A \sqcap_H B)(x)   =  \{t \in A(x) \cup B(x) \text{ $|$ }  t\leq \overline{A(x)}\wedge\overline{B(x)}\},
\end{array}
$$
for each $A,B\in \HFS(X)$ and any $x\in X$,
where, for any nonempty subset $S$ of $[0,1]$, we denote  by  $\underline{S}$ the infimum of  $S$ (in the natural order) and, by  $\overline{S}$, its supremum. In particular, $S\subseteq [\underline{S}, \overline{S}]$. Actually, these operators are derived pointwise from 
$$\begin{array}{l}
S \cup_H T   =  \{t \in S \cup T \text{ $|$ }  t\geq \underline{S}\vee \underline{T}\}
\text{ and } \\
S \cap_H T  =  \{t \in S \cup T \text{ $|$ }  t\leq \overline{S}\wedge \overline{T}\},
\end{array}$$
for any $S$ and $T$ nonempty subsets of $[0,1]$.
$\HFS$s have been successfully applied, for instance, to group decision making, decision support systems, computing with words or cluster analysis, see e.g. \cite{Qian2013, Liao2017, Rodriguez2012, Rodriguez2013, Rodriguez2014}.

From a mathematical point of view, the operators on $\HFS(X)$ preserve Zadeh's lattice operators on $\FS(X)$ and $\sqcup_I$ and  $\sqcap_I$ on $\IVFS(X)$. Nevertheless, $\HFS(X)$ is not a lattice, see \cite[Section 5.C]{Bustince/etal:2016} for a simple example. The reason is that $\cap_H$ is not a meet operator. On the contrary,  $\cup_H$ does provide a join-semilattice, see Theorem \ref{operatorclosed} below. Our aim now is to give a reformulation of  $\cap_H$. 
Let us first define a new partial order on $\SVFS(X)$. We shall manage the original definition of Grattan-Guinness in \cite{Grattan1976}, that is, a SVFS on $X$ is a map $X\to \powerset{[0,1]}\backslash \emptyset$. 

\begin{definition}\label{orderpowerset}
Let $S$ and $T$ be nonempty subsets of $[0,1]$, we say that $S \leq_S T$ if and only if:
\begin{enumerate}[$i)$]
\item $\overline{S}\leq \overline{T}$.
\item $\underline{S}\leq \underline{T}$.
\item $S\cap [\underline{T}, \overline{S}] \subseteq T$.
\end{enumerate}
\end{definition}

 \begin{prop}
 The relation $\leq_S$ is a partial order on $\powerset{[0,1]}\backslash \emptyset$.
 \end{prop}
 \begin{proof} 
 The relation is clearly reflexive. Let $S$ and $T$ be nonempty subsets of $[0,1]$ such that $S \leq_S T$ and $T \leq_S S$. Hence, $i)$ and $ii)$ in Definition \ref{orderpowerset} imply that $\underline{S}=\underline{T}$ and $\overline{S}=\overline{T}$. From $iii)$, $S\cap [\underline{T},\overline{S}]=S\cap [\underline{S},\overline{S}]=S\subseteq T$ and $T\cap [\underline{S},\overline{T}]=T\cap [\underline{T},\overline{T}]=T\subseteq S$, so $S=T$ and $\leq_S$ is antisymmetric.
 
For the transitive property, let $S$, $T$ and $U$ be nonempty subsets of $[0,1]$ such that $S \leq_S T$ and $T \leq_S U$. By $i)$ and $ii)$ in Definition \ref{orderpowerset}, $\underline{S}\leq \underline{T} \leq \underline{U}$ and $\overline{S}\leq \overline{T} \leq \overline{U}$. Now, 
$$S\cap [\underline{U},\overline{S}]\subseteq S\cap [\underline{T},\overline{S}]\subseteq T,$$ and hence,
$$S\cap [\underline{U},\overline{S}]\subseteq T\cap [\underline{U},\overline{S}]\subseteq T\cap [\underline{U},\overline{T}] \subseteq U.$$ 
Thus $S \leq_S U$, and the result follows.
 \end{proof}
From the partial order $\leq_S$, the class $\SVFS(X)$ inherits the following pointwise partial order.
\begin{definition}\label{orderSVFS}
For each $A,B\in \SVFS(X)$, we say $A \sqsubseteq_S B \iff A(x)\leq_S B(x) \text{ for any } x\in X.$
\end{definition}
It is easy to check that the partial order $\leq_S$ restricts to $\leq_I$ when working 
on closed intervals in $[0,1]$. Therefore, via the $\Hom(X,-)$ functor, we obtain the chain of poset maps \[\xymatrix{(\FS(X),\leq_F) \ar[r]^-{\Phi} & (\IVFS(X),\sqsubseteq_I) \ar[r]^-{i} & (\SVFS(X),\sqsubseteq_S),}\]
  which embeds $\FS$s into $\SVFS$s as those whose membership degrees are singletons. 
    \begin{definition} Let $S$ and $T$ be two nonempty subsets of $[0,1]$, we define:
$$
\begin{array}{l}
S\cup_S T=\left \{ 
  \begin{array}{ll}
  (S\cap [\underline{T},\overline{S}])\cup T& \text{if $\underline{S}\leq \underline{T}$,}\\
  (T\cap [\underline{S},\overline{T}])\cup S & \text{if $\underline{T}\leq \underline{S}$}\\
  \end{array}
  \right. \, \,\,\text{ and}\\
  \text{ } \\
  S\cap_S T=\left \{ 
  \begin{array}{ll}
  S\cap ([\underline{S},\underline{T}]\cup T) & \text{if $\underline{S}\leq \underline{T}$,}\\
  T\cap ([\underline{T},\underline{S}]\cup S) & \text{if $\underline{T}\leq \underline{S}$.}\\
  \end{array}
  \right.
  \end{array}$$
  \end{definition}

Observe that $\cup_S$ is the same operator as $\cup_H$. This is not the case for $\cap_S$ and $\cap_H$.
  
\begin{exm} 
Let $S=[0.3,0.7]$ and $T=\{0.4,0.5,0.6\}$, hence $S\cap_H T=[0.3,0.6]$ and $S\cap_S T=[0.3,0.4]\cup \{0.5,0.6\}$, see Figure \ref{Fighes}.
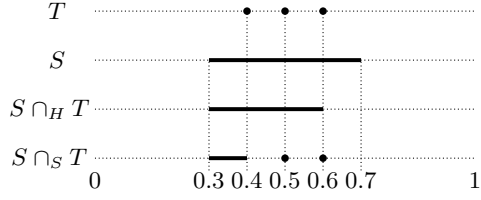
\begin{figure}[h!]
\begin{center}
\begin{tikzpicture}[yscale=0.65]
\draw[densely dotted] (0,0) -- (5,0);
\draw[densely dotted] (0,1) -- (5,1);
\draw[densely dotted] (0,-1) -- (5,-1);
\draw[densely dotted] (0,-2) -- (5,-2);
\draw[black, line width=1.5pt] (1.5,0.7);
\draw[black, line width=1.5pt] (1.5,0) -- (3.5,0);
\draw[black, line width=1.5pt] (1.5,-1) -- (3,-1);
\draw[black, line width=1.5pt] (1.5,-2) -- (2,-2);
\draw (2.5,-2) node [circle,fill=black, inner sep=1pt]{};
\draw (3,-2) node [circle,fill=black, inner sep=1pt]{};
\draw (2,1) node [circle,fill=black, inner sep=1pt]{};
\draw (2.5,1) node [circle,fill=black, inner sep=1pt]{};
\draw (3,1) node [circle,fill=black, inner sep=1pt]{};
\node at (-.5,0) {\scriptsize{$S$}};
\node at (-.5,1) {\scriptsize{$T$}};
\node at (-.6,-1) {\scriptsize{$S\cap_H T$}};
\node at (-.6,-2) {\scriptsize{$S\cap_S T$}};
\draw[densely dotted] (2,-2.25) -- (2,1);
\draw[densely dotted] (2.5,-2.25) -- (2.5,1);
\draw[densely dotted] (3,-2.25) -- (3,1);
\draw[densely dotted] (1.5,-2.25) -- (1.5,0);
\draw[densely dotted] (3.5,-2.25) -- (3.5,0);
\node at (0,-2.45) {\scriptsize{$0$}};
\node at (1.5,-2.45) {\scriptsize{$0.3$}};
\node at (2,-2.45) {\scriptsize{$0.4$}};
\node at (2.5,-2.45) {\scriptsize{$0.5$}};
\node at (3,-2.45) {\scriptsize{$0.6$}};
\node at (3.5,-2.45) {\scriptsize{$0.7$}};
\node at (5,-2.45) {\scriptsize{$1$}};
\end{tikzpicture}
\end{center}
\caption{Operators $\cap_H$ and $\cap_S$.}\label{Fighes}
\end{figure}
\end{exm}

 The operator $\cap_S$ is not closed in the class $\powerset{[0,1]}\backslash \emptyset$ as the following example shows.
  \begin{exm}\label{ex14}
Let us consider the subsets 
$$S=\left \{ \frac{1}{2n} \text{ $|$ } n\geq 1\right \} \text{ and } T=\left \{ \frac{1}{2n+1} \text{ $|$ }  n\geq 1\right \}.$$
Then $S\cap_S T=\emptyset$.
  \end{exm}

The reader might wonder if adding the empty set as a lower bound of all nonempty sets, the class $\powerset{[0,1]}$ becomes a lattice with $\cup_S$ and $\cap_S$. Notwithstanding, we may see that this also fails by an easy reformulation of the former example. Indeed, simply set $S=\{0\}$ and $T$ as in Example \ref{ex14}. Hence, $S\cap_S T=S$, but $S$ is not a lower bound of $T$. The underlying problem is of topological nature: there is a sequence inside $T$ whose limit does not belong to $T$. This suggests that the suitable class should be formed by closed sets under the standard topology on $[0,1]$. Let us denote by $\mathcal{C}([0,1])$ the class of all nonempty closed subsets of $[0,1]$.

\begin{thm}\label{operatorclosed}
The operators $\cup_S$ and $\cap_S$ endow the class $\mathcal{C}([0,1])$ with the lattice structure given by the order $\leq_S$.
\end{thm}
\begin{proof}
Let us fix $S$ and $T$ in $\mathcal{C}([0,1])$. Firstly, the operators are closed for $\mathcal{C}([0,1])$. By definition, $S\cup_S T$ and $S\cap_S T$ are constructed as a finite intersection and   union of closed sets, so they are also closed. Additionally, $S\cup_S T\not =\emptyset$, since $S$ or $T$ are contained in $S\cup_S T$, and $S\cap_S T\not =\emptyset$, since $\underline{S}$ or $\underline{T}$ are in $S\cap_S T$.

For brevity, denote $I=S\cup_S T$. We suppose that $\underline{S}\leq \underline{T}$, since the other case can be proved symmetrically. Therefore $I=(S\cap [\underline{T}, \overline{S}])\cup T$. 

The minimum $\underline{I}=\underline{T}$ since $\underline{T}$ is lower or equal than the minimum of $S\cap [\underline{T},\overline{S}]$, and $\underline{T}\in T$. The maximum $\overline{I}$ is the greatest among the maxima of the sets $S\cap [\underline{T},\overline{S}]$ and $T$, which are $\overline{S}$ and $\overline{T}$, respectively. So $\overline{I}=\overline{S}\vee \overline{T}$. Therefore $\underline{I}\geq \underline{S},\underline{T}$ and $\overline{I}\geq \overline{S},\overline{T}$. Now,
$$S\cap [\underline{I},\overline{S}]=S\cap [\underline{T},\overline{S}]\subseteq (S\cap [\underline{T},\overline{S}])\cup T=I$$
and 
$$T\cap [\underline{I},\overline{T}]=T\cap [\underline{T},\overline{T}]=T\subseteq (S\cap [\underline{T},\overline{S}])\cup T=I,$$
thus $S\leq_S I$ and $T\leq_S I$, so $I$ is an upper bound of both.
 
Let $J\in \mathcal{C}([0,1])$ such that $S\leq_S J$ and $T\leq_S J$. Then
$\underline{J}\geq \underline{T}= \underline{I}$ and $\overline{J}\geq \overline{S}\vee \overline{T}=\overline{I}$.
Therefore,
$$\begin{array}{rl}
I\cap [\underline{J},\overline{I}] & \subseteq I\cap [\underline{T},\overline{J}]\\
& = ((S\cap [\underline{T},\overline{S}])\cup T)\cap [\underline{T},\overline{J}]\\
&= (S\cap [\underline{T},\overline{S}] \cap [\underline{T},\overline{J}])\cup (T\cap [\underline{T},\overline{J}])\\
&\subseteq (S\cap [\underline{T},\overline{S}] \cap [\underline{S},\overline{J}])\cup J\\
& \subseteq ([\underline{T},\overline{S}]\cap J)\cup J\\
& \subseteq J.
\end{array}$$
Thus $I\leq_S J$, i.e. $I$ is the lowest upper bound.

Let us now denote $I=S\cap_S T$. Again, we suppose that $\underline{S}\leq \underline{T}$, so $I=S\cap ([\underline{S},\underline{T}]\cup T)$. The minimum of $I$ is $\underline{S}$, since $\underline{S}$ is the minimum of $[\underline{S},\underline{T}]\cup T$ and $\underline{S}\in S$.
Since $I\subseteq S$ and $I\subseteq [\underline{S},\underline{T}]\cup T$, hence $\overline{I}\leq \overline{S}$ and $\overline{I}\leq \overline{T}$. Now,
$$I\cap [\underline{S},\overline{I}]=I\cap [\underline{I},\overline{I}]=I=S\cap ([\underline{S},\underline{T}]\cup T)\subseteq S$$
and 
\[
\begin{split}
I\cap [\underline{T},\overline{I}] & =S\cap ([\underline{S},\underline{T}]\cup T)\cap [\underline{T},\overline{I}]\\
& =(S\cap [\underline{S},\underline{T}]\cap [\underline{T},\overline{I}])\cup (S\cap T\cap [\underline{T},\overline{I}])\\
& \subseteq \{\underline{T}\}\cup T\\
&=T,
\end{split}
\]
so $I$ is a lower bound of $S$ and $T$.

Let $J\in \mathcal{C}([0,1])$ such that $J\leq_S S$ and $J\leq_S T$. Therefore, $\underline{J}\leq \underline{S}=\underline{I}$. Before proving the condition on the maxima, let us show that  $J\cap [\underline{I},\overline{J}]\subseteq I$, with $\underline{I}=\underline{S}$. Since $J\leq_S S$, $J\cap [\underline{S},\overline{J}]\subseteq S$. We claim that $J\cap [\underline{S},\overline{J}]\subseteq [\underline{S},\underline{T}]\cup T$, and consequently, $J\cap [\underline{S},\overline{T}]\subseteq S\cap ([\underline{S},\underline{T}]\cup T)=I$. We distinguish three cases:
\begin{itemize}
\item If $\overline{J}< \underline{S}$, then $J\cap[\underline{S},\overline{J}]=\emptyset \subseteq  [\underline{S},\underline{T}]\cup T$.
\item If $\underline{S}\leq \overline{J} \leq \underline{T}$, then $J\cap[\underline{S},\overline{J}]\subseteq [\underline{S},\underline{T}]\subseteq [\underline{S},\underline{T}]\cup T$.
\item If $\underline{T}\leq \overline{J}$, then $J\cap[\underline{S},\overline{J}]=(J\cap [\underline{S},\underline{T}])\cup (J\cap [\underline{T},\overline{J}])\subseteq [\underline{S},\underline{T}]\cup T$, since $J\leq_S T$.
\end{itemize}

Finally, let $j\in J$. If $j\leq \underline{I}$ then $j\leq \overline{I}$. Otherwise, $j\in J\cap [\underline{I},\overline{J}]\subseteq I$, so $j\leq \overline{I}$. Thus $\overline{J}\leq \overline{I}$. This proves that $I$ is the greatest lower bound, and finishes the proof.
\end{proof}

\begin{definition}
A Closed-Valued Fuzzy Set (CVFS) $A$ on $X$ is a mapping $A:X\to \mathcal{C}([0,1])$. Then the set of all closed-valued fuzzy sets over $X$ is $\CVFS(X)=\Hom(X,\mathcal{C}([0,1])).$
\end{definition}

By (\ref{func}), $\CVFS(X)$ becomes a partially ordered set with the pointwise partial order inherited from $\leq_S$. Also, by (\ref{func2}), it is a lattice by means of the operators derived from $\cup_S$ and $\cap_S$. For consistency with the notation  established in the former sections, let us denote by $\leq_C$ the partial order $\leq_S$ when restricting to the class of closed sets. Then we mean by $\sqsubseteq_C$ the restriction of $\sqsubseteq_S$ to the class $\CVFS(X)$. Analogously, we shall use the symbols $\cup_C$ and $\cap_C$ instead of $\cup_S$ and $\cap_S$, respectively. Then, for each $A,B\in \CVFS(X)$, we define the operators $\sqcup_{C}$ and $\sqcap_{C}$ as:
\[
(A\sqcup_{C} B)(x)=A(x)\cup_C B(x) \text{ and}
(A\sqcap_{C} B)(x)=A(x)\cap_C B(x).
\]
for any $x\in X$.
\begin{cor}
The operators $\sqcap_C$ and $\sqcup_C$ endow the set $\CVFS(X)$ with the lattice structure associated to the partial order $\sqsubseteq_C$.
\end{cor}
\begin{proof}
It follows directly from Theorem \ref{operatorclosed}.
\end{proof}
The operators $\sqcup_C$ and $\sqcap_C$ equals $\sqcup_I$ and $\sqcap_I$, respectively, when restricting to $\IVFS$s.
Clearly, this is a consequence of the same property for the corresponding operators on nonempty closed sets and closed intervals in $[0,1]$. Indeed, let $[a,b]$ and $[c,d]$ be closed intervals with, for instance, $a\leq c$. Then 
\[
\begin{split}
[a,b]\cup_C [c,d] &=([a,b]\cap [c,b])\cup [c,d]\\
 & =[c,b]\cup [c,d]\\
 & =[c,b\vee d]\\
 & =[a,b]\cup_I [c,d].
 \end{split}
 \]
  Similarly, one may prove that $[a,b]\cap_C [c,d]=[a,b]\cap_I [c,d]$. Therefore, collecting the results of this section, we have proved the following
 
\begin{thm}
The sequence of maps
$$\xymatrix{(\FS(X),\cup_F,\cap_F) \ar[d]^-{\Phi}  & & \\   (\IVFS(X),\sqcup_I,\sqcap_I) \ar[r]^-{i}  &(\CVFS(X),\sqcup_C,\sqcap_C)}$$
is a chain of lattice embeddings, where $i$ is lifted from the identification of a closed interval as a nonempty closed set.
\end{thm}

\begin{rmk}\label{remCVFS}
We may also substitute the map $\Phi$ by any of the lattice embeddings defined in Section \ref{IVFS}. Their compositions with $i$ yield a family of lattice extensions of Zadeh's union and intersection.
\end{rmk}
Although this solution to \cite[Remark 2]{Bustince/etal:2016} is not complete, it covers most of the practical applications of HFSs. As commented in \cite{Rodriguez2014}, these only make use of nonempty finite sets, i.e. the so-called Typical Hesitant Fuzzy Sets (THFSs) \cite{Bedregal2014, Bedregal2014b}. Clearly, THFSs are CVFSs. Additionally, the approach by CVFSs  allows to work with different philosophies of understanding the membership degrees. For instance, we may include in the same framework THFSs and IVFSs. One could handle a $\CVFS$ whose membership degrees are given by closed intervals for some elements, whilst, for others, they are given by finite sets.

An interesting problem is to find a negation on $\CVFS$s that extends any of the ones considered in \cite{Bedregal2010} for $\IVFS$s. A different approach to this problem may be inspired from \cite{Walker05}. There, a negation is proposed for type-2 fuzzy sets, although the lattice structure is lost.

\section{Type-2 Fuzzy Sets}\label{T2FS}

We finish the paper analyzing how to place the former notions into the class of Type-2 Fuzzy Sets (T2FSs). This type of fuzzy sets is defined by Zadeh in \cite{Zadeh1975}. Actually, there it is defined something more general, the class of type-$n$ fuzzy sets. Nevertheless, here we focus on $n=2$, since, as far as our knowledge, no practical application has been developed for $n>2$. A T2FS is a  fuzzy set in which the membership degrees are given by FSs on $[0,1]$. Formally speaking,

\begin{definition}
A type-2 fuzzy set $A$ on \(X\) is a mapping $A : X \to \cerouno^\cerouno$, that is, the class of all $\T2FS$s on $X$ is
\[
\T2FS(X) = \Hom\big(X,\cerouno^\cerouno\big)\cong \Hom(X\times [0,1],[0,1]),
\] 
by the natural isomorphism (\ref{adjunction}). Observe that, looking at the right part of the isomorphism, a $\T2FS$ can be interpreted as a time-varying membership degree assignation, i.e. for any element $x\in X$ and time $t\in [0,1]$, we assign a membership degree $A(x,t)\in [0,1]$.
\end{definition}

\subsection{Set embeddings}

A way of viewing $\SVFS$s inside $\T2FS$s may be gotten making use of the obvious lattice inclusion $\dos \to [0,1]$, which maps $0\mapsto 0$ and $1\mapsto 1$. This can be lifted to $$\mu:\Hom([0,1],\dos)\to \Hom([0,1],[0,1]),$$ by means of the functor $\Hom([0,1],-)$, which maps a set $S\subseteq [0,1]$ to its characteristic function. 
 Hence, applying the functor $\Hom(X,-)$, we get $\overline{\mu}:\SVFS(X)\to \T2FS(X)$, where defined by 
\[\overline{\mu}(A)(x)(t)=\left \{\begin{array}{ll} 1 & \text{if $t\in A(x)$,}\\ 0 & \text{otherwise.} \end{array} \right.\]
for each $A\in \SVFS(X)$ and any $x\in X$.

Now, any injection from $\FS(X)$ into $\SVFS(X)$ provides, by composition with $\overline{\mu}$,  an embedding from $\FS(X)$ into $\T2FS(X)$. For instance, the composition $$\xymatrix{\FS(X) \ar[r]^-{\Phi} \ar@/_10pt/[rr]_-{\overline{\Phi}} & \SVFS(X) \ar[r]^-{\overline{\mu}} & \T2FS(X),}$$ where $\Phi$ is the map defined in Section \ref{IVFS}, yields the embedding that sees $\FS$s as $\T2FS$s whose membership degrees are singleton maps. That is, 
\[\overline{\Phi}(A)(x)(t)=\left \{\begin{array}{ll} 1 & \text{if $t= A(x)$,}\\ 0 & \text{otherwise,} \end{array} \right.\]
for each $A\in \FS(X)$ and any $x\in X$, see Figure \ref{Fig4}(a).

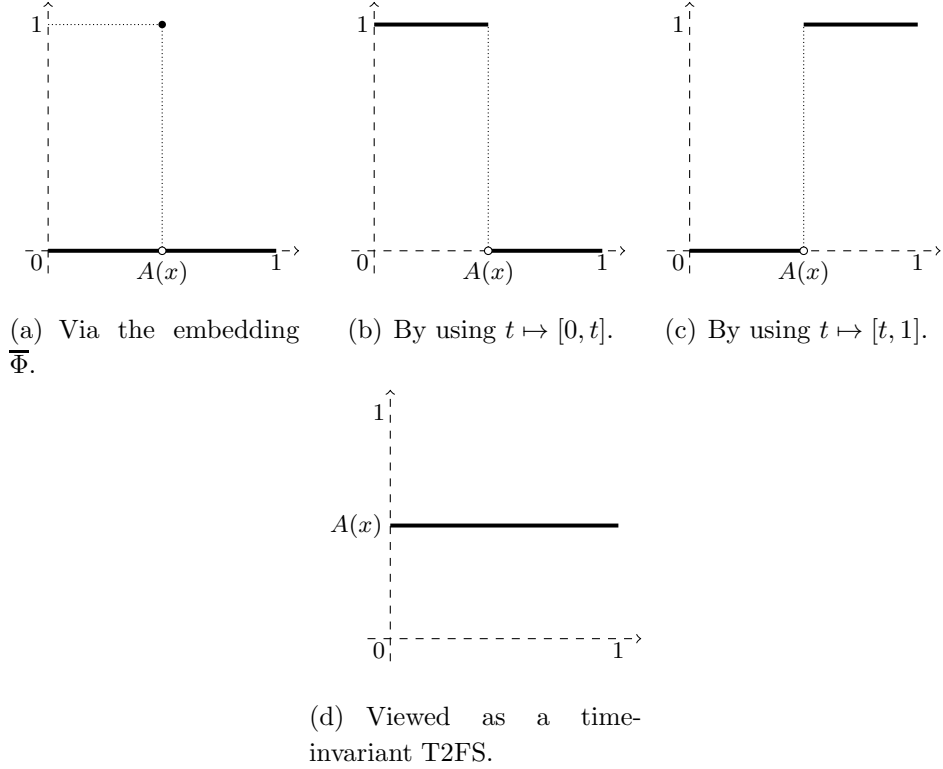
\begin{figure*}[tb]
\begin{center}
\subfigure[Via the embedding $\overline{\Phi}$.]{
\begin{tikzpicture}[scale=3]
\draw[dashed, ->] (0,-0.1) -- (0,1.1); 
\draw[dashed, ->] (-.1,0) -- (1.1,0); 
\draw[densely dotted] (0,1) -- (0.5,1);
\draw[densely dotted] (0.5,0) -- (0.5,1);
\node at (-.05,1) {\scriptsize{$1$}};
\node at (0.5,-0.1) {\scriptsize{$A(x)$}};
\node at (1,-0.05) {\scriptsize{$1$}};
\node at (-.05,-0.05) {\scriptsize{$0$}};
\draw[black, line width=1.5pt] (0,0) -- (1,0);
\draw (0.5,1) node [circle,fill=black, inner sep=1pt]{};
\filldraw (0.5,0) node [circle,fill=white, draw=black,  inner sep=1pt]{};
\end{tikzpicture}} \hspace{0.2cm}
\subfigure[By using $t\mapsto \lbrack 0, t\rbrack$.]{
\begin{tikzpicture}[scale=3]
\draw[densely dotted] (0.5,0) -- (0.5,1);
\draw[dashed, ->] (0,-0.1) -- (0,1.1); 
\draw[dashed, ->] (-.1,0) -- (1.1,0); 
\node at (-.05,1) {\scriptsize{$1$}};
\node at (0.5,-0.1) {\scriptsize{$A(x)$}};
\node at (1,-0.05) {\scriptsize{$1$}};
\node at (-.05,-0.05) {\scriptsize{$0$}};
\draw[black, line width=1.5pt] (0.5,0) -- (1,0);
\draw[black, line width=1.5pt] (0,1) -- (0.5,1);
\filldraw (0.5,0) node [circle,fill=white, draw=black,  inner sep=1pt]{};
\end{tikzpicture}}\hspace{0.2cm}
\subfigure[By using $t\mapsto \lbrack t, 1\rbrack$.]{
\begin{tikzpicture}[scale=3]
\draw[densely dotted] (0.5,0) -- (0.5,1);
\draw[dashed, ->] (0,-0.1) -- (0,1.1); 
\draw[dashed, ->] (-.1,0) -- (1.1,0); 
\node at (-.05,1) {\scriptsize{$1$}};
\node at (0.5,-0.1) {\scriptsize{$A(x)$}};
\node at (1,-0.05) {\scriptsize{$1$}};
\node at (-.05,-0.05) {\scriptsize{$0$}};
\draw[black, line width=1.5pt] (0.5,1) -- (1,1);
\draw[black, line width=1.5pt] (0,0) -- (0.5,0);
\filldraw (0.5,0) node [circle,fill=white, draw=black,  inner sep=1pt]{};
\end{tikzpicture}}\hspace{0.2cm}
\subfigure[Viewed as a time-invariant $\T2FS$.]{
\begin{tikzpicture}[scale=3]
\draw[dashed, ->] (0,-0.1) -- (0,1.1); 
\draw[dashed, ->] (-.1,0) -- (1.1,0); 
\node at (-.05,1) {\scriptsize{$1$}};
\node at (-0.15,0.5) {\scriptsize{$A(x)$}};
\node at (1,-0.05) {\scriptsize{$1$}};
\node at (-.05,-0.05) {\scriptsize{$0$}};
\draw[black, line width=1.5pt] (0,0.5) -- (1,0.5);
\node at (0.5,-0.1) {\scriptsize{$\phantom{A(x)}$}};
\end{tikzpicture}}
\end{center}
\caption{Membership degrees of a fuzzy set $A$ as a type-2 fuzzy set.}\label{Fig4}
\end{figure*}

We may also see $\FS(X)$ inside $\T2FS(X)$ through any lattice injection $\FS(X)\to \IVFS(X)$. For instance, consider $\overline{\Lambda}:\FS(X) \to \T2FS(X)$ defined by
\[\overline{\Lambda}(A)(x)(t)=\left \{\begin{array}{ll} 1 & \text{if $t\leq A(x)$,}\\ 0 & \text{otherwise,} \end{array} \right.\]
for any $A\in \FS(X)$ and $x\in X$. Observe that here we have used the map $\Lambda:\FS(X)\to \IVFS(X)$ of Section \ref{IVFS}, which is derived from the map $[0,1]\to \mathcal{I}([0,1])$ given by $t\mapsto [0,t]$. So that we have the commutative diagram
\begin{equation}\label{eq1}
\xymatrix{ \FS(X) \ar[rr]^-{\overline{\Lambda}} \ar[d]^-{\Lambda} \ar[rd]^-{\Theta}&   & \T2FS(X) \\
                      \IVFS(X) \ar[r]^-{i}&  \SVFS(X). \ar[ru]^-{\overline{\mu}} & }
\end{equation}
Albeit any of those described in Section \ref{IVFS} could be taken, see Figure \ref{Fig4}(b) and (c). Observe that we may also substitute $i:\IVFS(X)\to \SVFS(X)$ by the map $\Xi$ defined in Section \ref{SVFSb}.

A different methodology can be followed by taking into account the philosophy of seeing a T2FS as a time-varying FS. The construction is dual to the ones developed in this paper. Fix a set map $f:A\to B$ and apply the contravariant functor $\Hom(-,[0,1])$. Then, any map $h\in \Hom(B,[0,1])$ yields  a map $h\circ f\in \Hom(A,[0,1])$.
In this particular case, consider the projection to the first coordinate $p_1:X\times [0,1] \to X$ defined by $(x,t)\mapsto x$ for any $x\in X$ and $t\in [0,1]$. Therefore, there is a map $\Gamma : \FS(X) \to \T2FS(X)$
defined as
\begin{equation}\label{Gammamap}
\Gamma(A)(x)(t)=A(x)
\end{equation}
for each $A\in \FS(X)$ and any $x\in X$ and $t\in [0,1]$, see Figure \ref{Fig4}(d).

\subsection{Lattice embeddings}

Consider now the lattice structures. As for $\SVFS$s, some authors have treated the problem of endowing the class $\T2FS(X)$ with a lattice structure whose restriction to $\FS(X)$, via $\overline{\Phi}$, becomes Zadeh's union and intersection. In  \cite{Dubois1979} and \cite{Mizumoto1976}  it is proposed the lattice structure given by
$$
(A\sqcup_{T2}B)(x)=A(x)\cup_F B(x) \text{ and }
(A\sqcap_{T2}B)(x)=A(x)\cap_F B(x),
$$
for each $A,B\in \T2FS(X)$ and $x\in X$. Nevertheless, although $(\T2FS(X),\sqcup_{T2},\sqcap_{T2})$ is a complete lattice, as proved in \cite[Proposition 4.1]{Bustince/etal:2016}, Zadeh's operators on FSs are not recovered, see \cite{Dubois1979}, or \cite[Section IV.B]{Bustince/etal:2016} for a simple counterexample. 
Since it can be proved that $$\overline{\mu}:(\SVFS(X),\sqcup,\sqcap)\to (\T2FS(X),\sqcup_{T2},\sqcap_{T2})$$ is a lattice embedding, the underlying problem is that $\Phi$ is not a lattice map. Hence, replacing $\Phi$ by any other lattice embedding from $(\FS(X),\cup_F,\cap_F)$ to $(\SVFS(X),\sqcup,\sqcap)$, we may produce a family of lattice embeddings which preserve Zadeh's operators.  We have then proved  the following result.

\begin{thm}
The composition of maps
$$\xymatrix{(\FS(X),\cup_F,\cap_F) \ar[d]^-{\square} \ar@{.>}[r]& (\T2FS(x),\sqcup_{T2},\sqcap_{T2}) \\ (\IVFS(X),\sqcup_I,\sqcap_I) \ar[r]^-{\Xi} & (\SVFS(X),\sqcup,\sqcap)\ar[u]^-{\overline{\mu}} ,}$$
where $\square$ is any of the embeddings in Section \ref{IVFS}, provides a family of lattice embeddings from $\FS(X)$ to $\T2FS(X)$.
\end{thm}
Additionally, we may prove the following theorem.
\begin{thm}
The map 
$\Gamma : (\FS(X),\cup_F,\cap_F) \to (\T2FS(X),\sqcup_{T2},\sqcap_{T2})$, 
defined in \eqref{Gammamap}, is a lattice embedding. 
\end{thm}
\begin{proof}
Since $\Hom$ functors preserve monomorphisms, by construction, $\Gamma$ is one-to-one. Now,
\[
\begin{split}
(\Gamma(A)\cup_{T2} \Gamma(B))(x)(t) & =\Gamma(A)(x)(t)\vee \Gamma(B)(x)(t)\\
& =A(x)\vee B(x)\\
& =(A\cup_F B)(x)\\
& =\Gamma(A\cup_F B)(x)(t),
\end{split}
\]
for any $A,B\in \FS(X)$, $x\in X$ and $t\in [0,1]$. Similarly, $\Gamma(A)\cap_{T2} \Gamma(B)=\Gamma(A\cap_F B)$.
So $\Gamma$ is a lattice map.
\end{proof}

We may provide a different solution. In Section \ref{CVFS} we treat the problem of seeing $\FS$s as $\SVFS$s by means of the so-called $\CVFS$s. Here we give an embedding from $\CVFS(X)$ into $\T2FS(X)$. Consequently, its composition with any of the maps described in Remark \ref{remCVFS} yields a lattice extension of $\FS$s as $\T2FS$s preserving Zadeh's union and intersection.
Let us define a map $\delta:\mathcal{C}([0,1])\to [0,1]^{[0,1]}$ as follows. For any nonempty closed set $C$,
$$\delta(C)(t)=\left \{\begin{array}{ll} \underline{C} & \text{if $t\leq \underline{C}$,} \\
t & \text{if $t> \underline{C}$ and $t\in C$,}\\
0 & \text{if $t> \underline{C}$ and $t\notin C$,}
\end{array} \right.$$
for any $t\in [0,1]$. The map $\delta$ is injective. Indeed, given two closed sets $C$ and $D$, if $\delta(C)=\delta(D)$, then $t\in C$ if and only if $t=\delta(C)(t)=\delta(D)(t)$ if and only if $t\in D$.

\begin{rmk}
The map $\delta$ is well-defined when applying to arbitrary nonempty subsets. However, in this case, it is no longer injective. For instance, the sets $S=\{ 1/n \text{ with } n\geq 1\}$ and $T=S\cup\{0\}$ share the same image under $\delta$.
\end{rmk}

\begin{prop}\label{pro23}
The map $\delta:(\mathcal{C}([0,1]), \cup_C,\cap_C)\to ([0,1]^{[0,1]},\cup_F,\cap_F)$  is a lattice embedding.
\end{prop}
\begin{proof}
Let $C$ and $D$ be nonempty closed sets in $[0,1]$.  We prove first that $\delta(C\cup_C D)=\delta(C)\cup_{F}\delta(D)$. Without loss of generality, we may suppose that $\underline{C}\leq \underline{D}$. Then $C\cup_C D=(C\cap [\underline{D},\overline{C}])\cup D$ and $\underline{C\cup_C D}=\underline{D}$.  Observe   that
\[
\begin{split}
[\underline{D},\overline{C}\vee \overline{D}]\cap (C\cup D)  
  & =([\underline{D},\overline{C}\vee \overline{D}]\cap C)\cup  D\\
  & = ([\underline{D},\overline{C}]\cap C)\cup  D \\
  & = C\cup_C D.
\end{split}
\]
Then, whenever $\underline{D}\leq t \leq \overline{C}\vee \overline{D}$, $t\in C\cup D$ if and only if $t\in C\cup_C D$. Hence, for any $t\in [0,1]$,
\[
\begin{split}
\delta(C\cup_C D)(t)& =\left \{\begin{array}{ll} \underline{D} & \text{if $t\leq \underline{D}$,} \\
t & \text{if $\underline{D} <t$ and $t\in C\cup_CD$,}\\
0 & \text{if $\underline{D} <t$ and $t\notin C\cup_CD$}\\
\end{array} \right.\\
&  =\left \{\begin{array}{ll} \underline{D} & \text{if $t\leq \underline{D}$,} \\
t & \text{if $\underline{D} <t$ and $t\in C\cup D$,}\\
0 & \text{if $\underline{D} <t$ and $t\notin C\cup D$}\\
\end{array} \right.\\
& = \delta(C)(t) \vee \delta(D)(t)\\
& = (\delta(C)\cup_F \delta(D))(t).
\end{split}
\]
The equality $\delta(C\cap_C D)=\delta(C)\cap_{F}\delta(D)$ can be proved similarly.
\end{proof}
Therefore, by applying (\ref{func2}), we get the injective lattice map 
$$\Delta:(\CVFS(X),\sqcup_C,\sqcap_C) \to (\T2FS(X),\sqcup_{T2},\sqcap_{T2})$$ defined pointwise from $\delta$. Therefore, from Proposition \ref{pro23}, we obtain the following.\begin{cor}\label{corT2}
The composition of lattice maps 
$$\xymatrix{(\FS(X),\cup_F,\cap_F) \ar[d]^-{\square} \ar@{.>}[r] & (\T2FS(x),\sqcup_{T2},\sqcap_{T2}) \\(\IVFS(X),\sqcup_I,\sqcap_I) \ar[r]^-{i} & (\CVFS(X),\sqcup_C,\sqcap_C)\ar[u]^-{\Delta},}$$
where $\square$ is any of the embeddings in Section \ref{IVFS}, provides a family of lattice embeddings from $\FS(X)$ to $\T2FS(X)$.
\end{cor}

\begin{rmk}
Let $f:[0,1]\to [0,1]$ be a strictly increasing map with $f(0)=0$ and $f(1)=1$. We may modify $\delta$ and define $\delta_f:\mathcal{C}([0,1])\to [0,1]^{[0,1]}$ as, for any $C\in \mathcal{C}([0,1])$,
$$\delta_f(C)(t)=\left \{\begin{array}{ll} f(\underline{C}) & \text{if $t\leq \underline{C}$,} \\
f(t) & \text{if $t> \underline{C}$ and $t\in C$,}\\
0 & \text{if $t> \underline{C}$ and $t\notin C$,}
\end{array} \right.$$
for any $t\in [0,1]$. This is also an injective lattice map from $(\mathcal{C}([0,1]), \cup_C,\cap_C)$ to $([0,1]^{[0,1]},\cup_F,\cap_F)$. Therefore it also defines a lattice embedding $\Delta_f$ from $(\CVFS(X),\sqcup_C,\sqcap_C)$ to $(\T2FS(X),\sqcup_{T2},\sqcap_{T2})$. Thus, Corollary \ref{corT2} also remains valid if we change $\Delta$ by $\Delta_f$.
\end{rmk}
In Figure \ref{FigT2} we show examples of how a $\FS$ can be viewed as a $\T2FS$ by means of some of these embeddings.

\begin{figure*}[tb]
\begin{center}
\subfigure[$\square(A)(x)=\{A(x)\}$]{
\begin{tikzpicture}[scale=2.9]
\draw[densely dotted] (0.5,0) -- (0.5,0.5);
\draw[dashed, ->] (0,-0.1) -- (0,1.1); 
\draw[dashed, ->] (-.1,0) -- (1.1,0); 
\node at (-.05,1) {\scriptsize{$1$}};
\node at (0.5,-0.1) {\scriptsize{$A(x)$}};
\node at (-0.1,0.5) {\scriptsize{$A(x)$}};
\node at (1,-0.05) {\scriptsize{$1$}};
\node at (-.05,-0.05) {\scriptsize{$0$}};
\draw[black, line width=1.5pt] (0.5,0) -- (1,0);
\draw[black, line width=1.5pt] (0,0.5) -- (0.5,0.5);
\filldraw (0.5,0) node [circle,fill=white, draw=black,  inner sep=1pt]{};
\end{tikzpicture}}\hspace{0.2cm}
\subfigure[$\square(A)(x)=\lbrack 0,A(x)\rbrack$]{
\begin{tikzpicture}[scale=2.9]
\draw[densely dotted] (0.5,0) -- (0.5,0.5);
\draw[densely dotted] (0,0.5) -- (0.5,0.5);
\draw[dashed, ->] (0,-0.1) -- (0,1.1); 
\draw[dashed, ->] (-.1,0) -- (1.1,0); 
\node at (-.05,1) {\scriptsize{$1$}};
\node at (0.5,-0.1) {\scriptsize{$A(x)$}};
\node at (-0.1,0.5) {\scriptsize{$A(x)$}};
\node at (1,-0.05) {\scriptsize{$1$}};
\node at (-.05,-0.05) {\scriptsize{$0$}};
\draw[black, line width=1.5pt] (0.5,0) -- (1,0);
\draw[black, line width=1.5pt] (0,0) -- (0.5,0.5);
\filldraw (0.5,0) node [circle,fill=white, draw=black,  inner sep=1pt]{};
\end{tikzpicture}}
\hspace{0.2cm}
\subfigure[$\square(A)(x)=\lbrack A(x) , 1\rbrack$]{
\begin{tikzpicture}[scale=2.8]
\draw[densely dotted] (0.5,0) -- (0.5,0.5);
\draw[densely dotted] (1,0) -- (1,1);
\draw[densely dotted] (0,1) -- (1,1);
\draw[dashed, ->] (0,-0.1) -- (0,1.1); 
\draw[dashed, ->] (-.1,0) -- (1.1,0); 
\node at (-.05,1) {\scriptsize{$1$}};
\node at (0.5,-0.1) {\scriptsize{$A(x)$}};
\node at (-0.1,0.5) {\scriptsize{$A(x)$}};
\node at (1,-0.05) {\scriptsize{$1$}};
\node at (-.05,-0.05) {\scriptsize{$0$}};
\draw[black, line width=1.5pt] (0,0.5) -- (0.5,0.5);
\draw[black, line width=1.5pt] (0.5,0.5) -- (1,1);
\end{tikzpicture}}
\hspace{0.2cm}
\subfigure[$\square(A)(x)=\lbrack A(x) , b(x) \rbrack$ \newline $b(x)=\log_2(A(x)+1)$ ]{
\begin{tikzpicture}[scale=2.8]
\draw[densely dotted] (0.5,0) -- (0.5,0.5);
\draw[densely dotted] (0.75,0) -- (0.75,0.75);
\draw[densely dotted] (0,0.75) -- (0.75,0.75);
\draw[dashed, ->] (0,-0.1) -- (0,1.1); 
\draw[dashed, ->] (-.1,0) -- (1.1,0); 
\node at (-.05,1) {\scriptsize{$1$}};
\node at (0.5,-0.1) {\scriptsize{$A(x)$}};
\node at (-0.1,0.5) {\scriptsize{$A(x)$}};
\node at (0.75,-0.1) {\scriptsize{$b(x)$}};
\node at (-0.1,0.75) {\scriptsize{$b(x)$}};
\node at (1,-0.05) {\scriptsize{$1$}};
\node at (-.05,-0.05) {\scriptsize{$0$}};
\draw[black, line width=1.5pt] (0,0.5) -- (0.5,0.5);
\draw[black, line width=1.5pt] (0.5,0.5) -- (0.75,0.75);
\draw[black, line width=1.5pt] (0.75,0) -- (1,0);
\filldraw (0.75,0) node [circle,fill=white, draw=black,  inner sep=1pt]{};
\end{tikzpicture}}
\end{center}
\caption{Membership degrees of a fuzzy set $A$ viewed as a type-2 fuzzy set through some of the compositions of Corollary \ref{corT2}.}\label{FigT2}
\end{figure*}
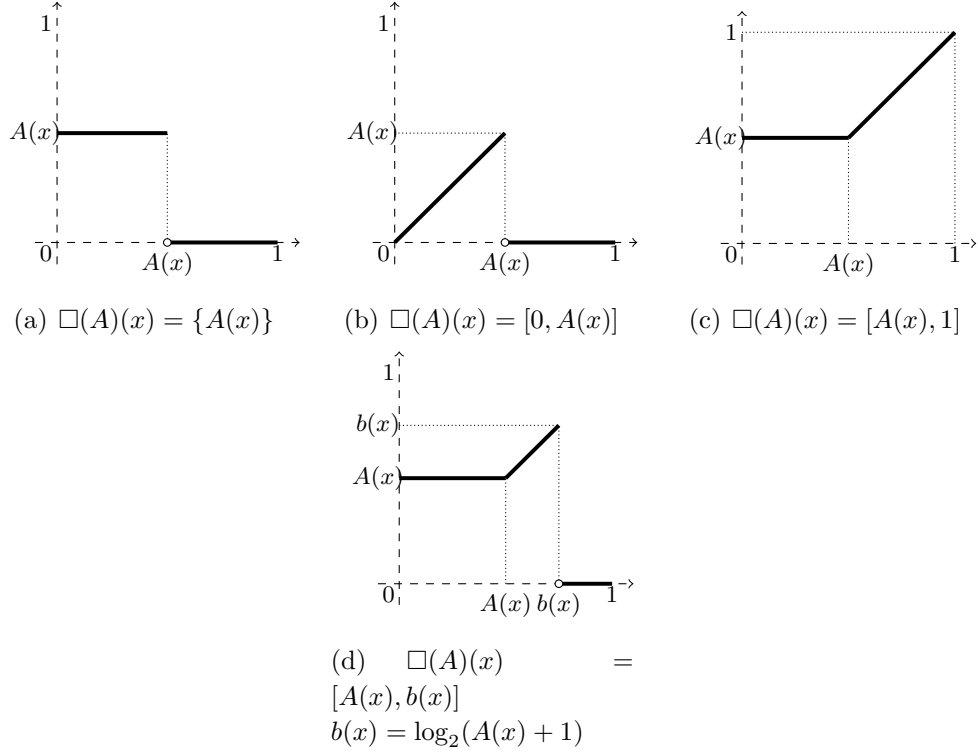

\section{Conclusion}
In this paper we have analyzed the relationships, as lattices, of the following types of fuzzy sets: $\FS$s, $\IVFS$s, $\SVFS$s and $\T2FS$s. The use of the language of category theory has allowed us to understand the underlying mathematical problem of inserting $\FS$s inside $\SVFS$s and $\T2FS$s. The solutions that we have shown are then in accord with this viewpoint. It  seems more convenient, mathematically speaking, to exchange the way of seeing $\FS$s as $\SVFS$ or $\T2FS$ rather than substituting the standard lattice structures. We think that this way of reasoning may be extrapolated to other models of handling uncertainty and imperfect information.

On the other hand, the success of $\HFS$s suggests that, in some particular situations, it may be convenient to modify the usual lattice structures. In this sense, we have adjusted the meet operator defined by Torra for $\HFS$s and defined $\CVFS$s. This generalization covers most of the applications of $\HFS$s, since it may handle $\THFS$s and $\IVFS$s. Additionally, $\CVFS$s over a fixed universe set become a lattice.

\end{document}